\documentclass{article}

\usepackage{arxiv}

\usepackage[utf8]{inputenc} 
\usepackage[T1]{fontenc}    
\usepackage{hyperref}       
\usepackage{url}            
\usepackage{booktabs}       
\usepackage{amsfonts}       
\usepackage{nicefrac}       
\usepackage{microtype}      
\usepackage{graphicx}
\usepackage[numbers,compress]{natbib}
\usepackage{doi}
\usepackage{amsthm}
\usepackage{amsmath}

\newtheorem{theorem}{Theorem}
\newtheorem{corollary}{Corollary}

\newtheorem{proposition}{Proposition}
\newtheorem{conjecture}{Conjecture}
\newtheorem{remark}{Remark}
\newtheorem*{notation}{Notation}

\title{The Geometric Cost of Normalization: Affine Bounds on the Bayesian Complexity of Neural Networks}

\author{Sungbae Chun \\
	Department of Mathematics\\
	Universität Bonn\\
	\texttt{sungbae.chun@uni-bonn.de}
}

\renewcommand{\shorttitle}{The Geometric Cost of Normalization}

\hypersetup{
pdftitle={The Geometric Cost of Normalization},
pdfsubject={cs.LG, stat.ML},
pdfauthor={Sungbae Chun},
pdfkeywords={Deep Learning, Singular Learning Theory, Normalization, Bayesian Complexity},
}

\begin{document}
\maketitle

\begin{abstract}
\texttt{LayerNorm} and \texttt{RMSNorm} impose fundamentally different geometric constraints on their outputs — and this difference has a precise, quantifiable consequence for model complexity. We prove that \texttt{LayerNorm}'s mean-centering step, by confining data to a linear hyperplane (through the origin), reduces the Local Learning Coefficient (LLC) of the subsequent weight matrix by exactly $m/2$ (where $m$ is its output dimension); \texttt{RMSNorm}'s projection onto a sphere preserves the LLC entirely. This reduction is structurally guaranteed before any training begins, determined by data manifold geometry alone.

The underlying condition is a geometric threshold: for the codimension-one manifolds we study, the LLC drop is binary --- any non-zero curvature, regardless of sign or magnitude, is sufficient to preserve the LLC, while only affinely flat manifolds cause the drop. At finite sample sizes this threshold acquires a smooth crossover whose width depends on how much of the data distribution actually experiences the curvature, not merely on whether curvature exists somewhere.

We verify both predictions experimentally with controlled single-layer scaling experiments using the wrLLC framework~\cite{wrLLC}. We further show that \texttt{Softmax} simplex data introduces a ``smuggled bias'' that activates the same $m/2$ LLC drop when paired with an explicit downstream bias, proved via the affine symmetry extension of the main theorem and confirmed empirically.
\end{abstract}


\section{Introduction}

\texttt{LayerNorm} and \texttt{RMSNorm} are widely treated as interchangeable design choices, but they differ in a structural way that has not previously been quantified: \texttt{LayerNorm} reduces the effective parameter count of any subsequent weight matrix, while \texttt{RMSNorm} does not. We prove this using the Local Learning Coefficient (LLC)~\cite{Watanabe2009, devinterp}, a Bayesian measure of a model's effective complexity: \texttt{LayerNorm}'s mean-centering step reduces LLC by exactly $m/2$ per normalized layer, a reduction that is structurally guaranteed before training begins. \texttt{RMSNorm} incurs no such cost.

The mechanism is geometric. Normalization layers~\cite{Ioffe2015, Ba2016, Zhang2019} constrain data to a submanifold of $\mathbb{R}^d$, and directions of the weight matrix orthogonal to that manifold become invisible to the loss --- genuine symmetries that reduce the LLC. \emph{What property of the manifold determines whether this drop occurs?} Our answer is a geometric one: whether the manifold is affinely flat. We prove this for single linear layers with squared loss and find experimentally that the threshold is binary --- any non-zero curvature, regardless of sign or magnitude, is sufficient to preserve the LLC.

Our contributions are:

\begin{itemize}
    \item \textbf{A geometric threshold for LLC reduction.} We prove that a weight matrix $W \in \mathbb{R}^{m \times d}$ loses exactly $m/2$ from its LLC whenever the input data lies on a linear hyperplane (through the origin; linear span $d_s = d-1$) (Section~\ref{sec:theory}). Experimentally, we find that for codimension-one manifolds, any non-zero curvature is sufficient to prevent this drop (Section~\ref{sec:curvature}).

    \item \textbf{A finite-sample phase transition.} Experimentally, we find that the observed binary threshold acquires a smooth crossover at finite sample sizes. The width and location of the transition depend on how much of the data distribution actually experiences the curvature, not merely on whether curvature exists somewhere on the manifold (Section~\ref{sec:curvature}).

    \item \textbf{Exact LLC predictions for \texttt{LayerNorm} and \texttt{RMSNorm}.} \texttt{LayerNorm} confines data to a linear hyperplane (mean-centering sets $\mathbf{1}^\top x = 0$), reducing LLC by $m/2$. \texttt{RMSNorm} projects onto a sphere, which has full linear span, leaving LLC unchanged. We validate both predictions with controlled scaling experiments (Section~\ref{sec:normalization}).

    \item \textbf{Simplex-Bias Duality (Proposition~\ref{prop:simplex_duality}).} \texttt{Softmax} simplex data carries a hidden implicit bias ($b_\text{smuggled}$) that is inert in strictly linear layers ($\Delta\lambda = 0$) but activates an $m$-dimensional symmetry and drops LLC by $m/2$ whenever the layer includes an explicit bias (proved theoretically; measured $\Delta\lambda = 2.26 \approx m/2 = 2.0$, robust across seeds). The Post-\texttt{LayerNorm} case is an open problem: the theoretical mechanism is plausible but current SGLD estimation is unreliable for this configuration (Section~\ref{sec:theory}).
\end{itemize}

\section{Background}

\subsection{Singular Learning Theory and the LLC}
Singular Learning Theory (SLT), developed by Watanabe~\cite{Watanabe2009} and applied to deep learning by \citet{Wei2022}, measures a model's Bayesian complexity via the \textbf{Local Learning Coefficient} (LLC) --- the effective number of parameters as seen by the Bayesian posterior. When the loss has continuous symmetries (directions in parameter space that leave the loss unchanged), the LLC is strictly less than the nominal parameter count. We show that normalization layers introduce such symmetries structurally, before any training occurs.

\subsection{Related Work}
The geometry of the loss landscape and continuous symmetries in overparameterized networks has been studied from several angles. \citet{Simsek2021} characterize how permutation symmetries of the network architecture give rise to connected manifolds of equivalent optima, reducing LLC in a manner structurally analogous to our symmetry space $S_{\mathcal{M}}$. Our contribution is orthogonal: we identify a separate, \textit{data-manifold-induced} source of degeneracy that is independent of the weight-space symmetries of the architecture itself.

\citet{Gupta2024LN} study the geometric structure of \texttt{LayerNorm} and \texttt{RMSNorm} at inference time; our work differs in connecting normalization geometry to the LLC and deriving exact algebraic predictions.

The LLC has been applied to study training dynamics and phase transitions in Transformers~\cite{LLC_transformers} and to characterize degeneracy in singular models~\cite{devinterp}. Our work complements these by pinpointing a structural---rather than emergent---source of degeneracy arising from normalization operations.

\textbf{Attention sinks.}
\citet{Xiao2023} empirically identified sink tokens in autoregressive language models; \citet{RanMilo2026} proved their existence is necessary in softmax Transformers; \citet{Ruscio2025} offer a complementary geometric analysis of the sink phenomenon. Our framework suggests a structural interpretation: \texttt{LayerNorm}'s mean-centering reduces LLC($W_V$) by $m/2$, which may structurally bias $W_V$ toward low-rank solutions associated with sink behavior. We were unable to validate this prediction experimentally (see Conclusion); the connection between LLC reduction and sink formation remains an open question.

\subsection{Holonomic Constraints from Normalization}
Each normalization operation acts as a \emph{holonomic constraint} on the data manifold, confining its output to a specific geometric surface. This surface determines the linear span available to the next weight matrix, which in turn determines whether an LLC drop occurs.

\begin{itemize}
    \item \textbf{LayerNorm~\cite{Ba2016}} enforces zero mean and unit variance. The mean-centering step projects data onto the hyperplane $\{\mathbf{1}^\top x = 0\}$, a linear subspace of dimension $d-1$. The variance normalization further confines data to a sphere within this hyperplane, but the linear span of that sphere equals the span of the full hyperplane, so $d_s = d-1$ regardless.
    \item \textbf{RMSNorm~\cite{Zhang2019}} enforces unit RMS without mean-centering. Outputs lie on a sphere $\{\|x\| = \sqrt{d}\}$, which contains $\pm\sqrt{d}\,e_i$ for each standard basis direction $e_i$, so $\mathrm{span}(S^{d-1}) = \mathbb{R}^d$. Linear span: $d_s = d$.
\end{itemize}

The span reduction in \texttt{LayerNorm} is the precise condition for an LLC drop; \texttt{RMSNorm} avoids it. We observe in Section~\ref{sec:curvature} and Appendix~\ref{app:homology} that for the manifolds we test, non-zero curvature is sufficient to achieve full linear span, while affinely flat manifolds do not span the full ambient space.

\begin{remark}[Lagrangian Mechanics Analogy]
The term \textit{holonomic} is borrowed deliberately from Lagrangian mechanics. In that setting, holonomic constraints restrict a system's trajectory to an allowed configuration manifold, reducing the effective degrees of freedom of the dynamics. Normalization operations play an analogous role here: they confine the input data to a submanifold of $\mathbb{R}^d$, constraining which directions of $W$ are observable by the loss.

This should be distinguished from the \textit{symmetries of the Lagrangian}---invariances such as the $GL(r)$ scaling symmetry in adjacent matrix factorizations, which leave the loss unchanged regardless of the data geometry. In the Lagrangian analogy, these correspond to Noether symmetries: each continuous invariance produces a conserved (flat) direction in the loss landscape. The holonomic constraint cost and the Lagrangian symmetry cost are structurally independent sources of degeneracy. The wrLLC~\cite{wrLLC} methodology is precisely designed to isolate the former from the latter, freezing all components that carry Lagrangian symmetries so that the measured LLC drop reflects only the geometric cost of the constraint.
\end{remark}

\subsection{The Simplex, the Smuggled Bias, and Affine Symmetries}
Beyond variance-based normalizations, routing mechanisms like \texttt{Softmax} impose holonomic constraints of a different kind. \texttt{Softmax} projects arbitrary real vectors onto the standard simplex:
\[ \Sigma^{d-1} = \left\{ x \in \mathbb{R}^d \;\middle|\; \sum_{i=1}^d x_i = 1, \: x_i > 0 \right\} \]

\begin{remark}[Semi-algebraic nature of the Softmax constraint]
Strictly speaking, the Softmax simplex is defined by one equality ($\sum x_i = 1$) together with $d$ strict inequality constraints ($x_i > 0$), making it a semi-algebraic set rather than a holonomic constraint in the classical mechanics sense. The open simplex is nevertheless a smooth manifold, and the term \emph{holonomic} is used loosely throughout this work to mean any constraint that confines data to a smooth submanifold. For all results in this section, only $\mathrm{span}(\mathcal{M})$ matters, and the inequality constraints contribute nothing to this span: $\mathrm{span}(\Sigma^{d-1}) = \mathbb{R}^d$ (since $e_1, \ldots, e_d \in \mathrm{cl}(\Sigma^{d-1})$). All results therefore apply to the Softmax case without modification.
\end{remark}

Because the outputs are strictly constrained by $\mathbf{1}^T x = 1$, the data lies on a $(d-1)$-dimensional hyperplane. To rigorously analyze the degeneracy of a subsequent linear layer $W \in \mathbb{R}^{m \times d}$, we can apply an affine shift to center this hyperplane at the origin. Let $x' = x - \frac{1}{d}\mathbf{1}$. By construction, $\mathbf{1}^T x' = 0$, meaning $x'$ resides strictly in a $(d-1)$-dimensional linear span. 

Applying the linear layer $W$ to the original data yields:
\[ Wx = W\left(x' + \frac{1}{d}\mathbf{1}\right) = Wx' + \left(\frac{1}{d}W\mathbf{1}\right) \]
Let $b_{smuggled} = \frac{1}{d}W\mathbf{1}$ (the row-wise mean of $W$). Thus, $Wx = Wx' + b_{smuggled}$.

This geometric decomposition formalizes three architectural scenarios with distinct degeneracy consequences:

\begin{proposition}[Simplex-Bias Duality]
\label{prop:simplex_duality}
Let $f(x) = Wx$ with $W \in \mathbb{R}^{m \times d}$, and let the input lie on the standard simplex $\Sigma^{d-1}$, so $\mathbf{1}^\top x = 1$. Decompose $x = x' + \frac{1}{d}\mathbf{1}$ where $\mathbf{1}^\top x' = 0$; then $f(x) = Wx' + b_\textup{smuggled}$ where $b_\textup{smuggled} = \frac{1}{d}W\mathbf{1}$. The LLC drop $\Delta\lambda$ relative to a full-rank Gaussian baseline is:
\begin{enumerate}
    \item[(i)] \textbf{Strictly linear} ($y = Wx$, MSE loss): $\mathrm{span}(\Sigma^{d-1}) = \mathbb{R}^d$, so no symmetry space exists and $\Delta\lambda = 0$.
    \item[(ii)] \textbf{Affine layer} ($y = Wx + b$, MSE loss): For any $c \in \mathbb{R}^m$, the substitution $(W,\,b) \to (W + c\mathbf{1}^\top,\, b - c)$ leaves $y$ unchanged on $\Sigma^{d-1}$, since $b_\textup{smuggled}$ and $b$ are indistinguishable to the loss. This yields an $m$-dimensional symmetry space and, by Corollary~\ref{cor:affine_llc}, $\Delta\lambda = m/2$.
\end{enumerate}
\end{proposition}

\begin{conjecture}[Post-LayerNorm Smuggled-Bias Degeneracy]
\label{conj:postln}
Under the same setting, with $y = \mathrm{LN}(Wx)$ and MSE loss on the \texttt{LayerNorm} output: $\Delta\lambda = m/2$. The mechanism is that \texttt{LayerNorm}'s translation-invariance renders $b_\textup{smuggled}$ an effective blind spot, but the precise symmetry group of the Post-\texttt{LN} loss landscape has not been characterized. Current SGLD-based LLC estimation is unreliable in this configuration: multi-seed testing gives $\Delta\lambda = -0.92 \pm 0.86$ (5 seeds), indicating that the Gaussian baseline LLC is poorly calibrated for the non-affine \texttt{LN} loss landscape. Validating this conjecture requires either a direct proof or an estimator better calibrated to Post-\texttt{LN} models.
\end{conjecture}

The common thread is that \texttt{Softmax}'s affine geometry primes the network for degeneracy, with the precise mechanism determined by downstream operations.

\subsection{Main Result: Data-Manifold Induced RLCT Reduction}\label{sec:theory}
When a weight matrix $W \in \mathbb{R}^{m \times d}$ operates on data confined to a manifold $\mathcal{M}$, those directions of $W$ orthogonal to $\mathrm{span}(\mathcal{M})$ are invisible to the loss. This ``blind spot'' is the precise algebraic manifestation of the data-manifold induced singularity in the parameter space.

\begin{notation}[Symmetry Space]
Let $\mathcal{M} \subset \mathbb{R}^d$ and let $d_s = \dim(\mathrm{span}(\mathcal{M}))$. The \emph{symmetry space} of $\mathcal{M}$ is
\[ S_{\mathcal{M}} := \{ U \in \mathbb{R}^{m \times d} \mid Ux = 0 \;\forall\, x \in \mathcal{M} \}. \]
Since $S_{\mathcal{M}}$ is the kernel of the evaluation map $W \mapsto (Wx)_{x \in \mathcal{M}}$, which has rank $d_s$, we have $\dim(S_{\mathcal{M}}) = m(d - d_s)$ by rank-nullity.
\end{notation}

This setting admits a closed-form RLCT via Watanabe's formula: the loss is a quadratic polynomial, the zero-loss manifold $W_0$ is smooth and explicitly characterizable, and the Morse-Bott condition holds for any non-degenerate input distribution. The result is an exact algebraic prediction of normalization's Bayesian cost: for any normalization that acts as a span restriction, the LLC drop is $\Delta\lambda = m(d-d_s)/2$, determined entirely by the data geometry before any training occurs. We conjecture that the same reduction holds per-layer in deep networks under the wrLLC isolation protocol.

\begin{theorem}[Symmetry-Induced RLCT Reduction]
\label{lem:symmetry_llc}
Let $f(x) = Wx$ be a single linear layer with $W \in \mathbb{R}^{m \times d}$, let $L$ be the squared loss with a realizable teacher $W^*$, and let the input data lie on $\mathcal{M}$ with $d_s = \dim(\mathrm{span}(\mathcal{M}))$. Let $\lambda_0 = md/2$ denote the RLCT of the same model with full-rank input data.

\textbf{(i) Free action.} The symmetry space $S_{\mathcal{M}}$ acts freely on $W_0$ by translation: for any $U \in S_{\mathcal{M}}$ and $W \in W_0$,
\[
  (W + U)x = Wx \quad \forall\, x \in \mathcal{M},
\]
so $W + U \in W_0$.

\textbf{(ii) Zero-loss manifold.} $W_0 = W^* + S_{\mathcal{M}}$, so $\dim(W_0) = m(d - d_s)$ and $\mathrm{codim}(W_0) = m \cdot d_s$.

\textbf{(iii) RLCT.} The expected loss is $L(W) = \|(W - W^*)\Sigma^{1/2}\|_F^2$, a polynomial with smooth zero set $W_0$ and non-degenerate normal Hessian $\Sigma|_{\mathrm{span}(\mathcal{M})} \otimes I_m$ (the FIM restricted to $\mathrm{span}(\mathcal{M})$), which is non-degenerate for any non-degenerate input distribution. By~\cite{Watanabe2009}, the RLCT equals
\[
  \lambda = \frac{\mathrm{codim}(W_0)}{2} = \frac{m \cdot d_s}{2}.
\]

\noindent Consequently, the RLCT reduction relative to the full-rank baseline is
\[
  \Delta\lambda := \lambda_0 - \lambda = \frac{md}{2} - \frac{m \cdot d_s}{2} = \frac{m(d - d_s)}{2}.
\]
\end{theorem}

\begin{corollary}[LLC Drop for Normalization Layers]
\label{cor:llc_drop}
Let $W \in \mathbb{R}^{m \times d}$ be a single linear layer with squared loss and non-degenerate input distribution. If a preceding normalization restricts the input to $\mathcal{M}$ with $\dim(\mathrm{span}(\mathcal{M})) = d_s$, then
\[ \Delta\lambda = \frac{m(d - d_s)}{2}. \]
In particular: \texttt{LayerNorm} gives $d_s = d-1$ and $\Delta\lambda = m/2$; \texttt{RMSNorm} gives $d_s = d$ and $\Delta\lambda = 0$.
\end{corollary}

\begin{remark}[Extension to deep networks]
Theorem~\ref{lem:symmetry_llc} is proved for single linear layers with squared loss. We conjecture that the same RLCT reduction holds for each normalized weight matrix $W_l$ in a deep network, with the wrLLC methodology~\cite{wrLLC} providing the isolation mechanism: freezing all parameters except $W_l$ reduces the multi-layer problem to the single-layer case of Theorem~\ref{lem:symmetry_llc} locally. The experiments in Section~\ref{sec:normalization} confirm the single-layer theorem; they do not test this conjecture, as all architectures used are single linear layers in a teacher-student setup. Empirical validation in actual multi-layer models, and a formal proof requiring verification of the Morse-Bott condition for the composed map $g \circ \mathrm{norm} \circ W_l$, are left to future work.
\end{remark}

\begin{corollary}[Layer-wise Additivity under wrLLC (conditional)]
\label{cor:additivity}
Assume the conjecture in the Remark above holds. Let $f$ be a multi-layer network in which layer $l$ has weight matrix $W_l \in \mathbb{R}^{m_l \times d_l}$ preceded by a normalization restricting its input to $\mathcal{M}_l$ with $d_{s,l} = \dim(\mathrm{span}(\mathcal{M}_l))$. Under the wrLLC protocol---which freezes all parameters except $W_l$---the model locally reduces to the single-layer case of Corollary~\ref{cor:llc_drop}. Hence the LLC drop at layer $l$ is
\[
  \Delta\lambda_l = \frac{m_l(d_l - d_{s,l})}{2},
\]
independently of all other layers.
\end{corollary}

\begin{remark}
The motivation for the wrLLC protocol in Corollary~\ref{cor:additivity} differs from its original use in multi-layer networks \cite{wrLLC}. In the single-layer experiments of Section~\ref{sec:normalization}, there are no adjacent matrix factors and hence no $GL(r)$ scaling symmetry to eliminate; wrLLC simply identifies the target weight matrix $W$. In a multi-layer setting, freezing all other weights additionally removes the $GL(r)$ invariances inherent in adjacent factorizations, making the two isolation benefits structurally independent.
\end{remark}

\begin{remark}[wrLLC as a lower bound on degeneracy]
The wrLLC measurement provides a lower bound on the degeneracy that $W_l$ contributes to the full network. The argument is direct: any $U \in S_{\mathcal{M}_l}$ satisfies $Ux = 0$ for all $x \in \mathcal{M}_l$, so $(W_l + U)x = W_l x$ on all training data. Every downstream operation---nonlinearities, subsequent layers, the loss---sees the identical input and produces the identical output, so $U$ is a genuine symmetry of the full network loss, not only of the isolated layer. Joint effects such as $GL(r)$ scaling symmetries between adjacent layers can only add further symmetries on top. Hence the full network has at least $m_l(d_l - d_{s,l})$ degenerate directions attributable to $W_l$ alone. If $p_\text{total}$ denotes the total parameter count, the global RLCT satisfies
\[
  \lambda_\text{global} \leq \frac{p_\text{total} - m_l(d_l - d_{s,l})}{2},
\]
meaning the $m_l d_l$ degrees of freedom of $W_l$ contribute at most $\mathrm{wrLLC}(W_l) = m_l d_{s,l}/2$ to the global RLCT, rather than the naive $m_l d_l/2$.
\end{remark}

\subsection{Supporting Geometric Results}
The following results are mathematically straightforward, but we include proof sketches as these exact statements do not appear verbatim in the literature.

\begin{theorem}[Rank Manifold Codimension~{\cite[Ch.~1, \S4, Exercise~13]{GuilleminPollack1974}}]
\label{thm:rank_manifold}
The set $\mathcal{R}_r := \{ A \in \mathbb{R}^{m \times d} \mid \mathrm{rk}(A) = r \}$ is a smooth submanifold of $\mathbb{R}^{m \times d}$ of codimension $(m-r)(d-r)$.
\end{theorem}

\begin{proof}[Proof sketch]
The result is posed as an exercise in \cite{GuilleminPollack1974} with the hint to consider the Schur complement map $\Phi(A) := E - DB^{-1}C \in \mathbb{R}^{(m-r)\times(d-r)}$, where $B, C, D, E$ are the blocks of $A$ after permuting so that the top-left $r\times r$ block is invertible. Since $\partial\Phi/\partial E = \mathrm{Id}$, the differential $d\Phi$ is surjective, and the result follows by the regular value theorem.
\end{proof}

Combining the symmetry space dimension (Notation above) and Theorem~\ref{thm:rank_manifold}, we establish the dimension of the constrained weight manifold:

\begin{theorem}[Constrained Weight Manifold Dimension]
\label{thm:constrained_weight}
Let $\mathcal{M} \subset \mathbb{R}^d$ with $d_s = \dim(\mathrm{span}(\mathcal{M}))$, and let
\[ W_{\mathcal{M},r} := \{ W \in \mathbb{R}^{m \times d} \mid \mathrm{rk}(W) = r,\ Wx = 0\ \forall\, x \in \mathcal{M} \}. \]
Then $\dim(W_{\mathcal{M},r}) = r(d - d_s + m - r)$.
\end{theorem}

\begin{proof}[Proof sketch]
Choose a basis for $\mathbb{R}^d$ aligning $\mathrm{span}(\mathcal{M})$ with the first $d_s$ coordinate directions. The constraint $Wx = 0$ for all $x \in \mathcal{M}$ then reduces to $W_1 = 0$ (the first $d_s$ columns vanish), so $W_{\mathcal{M},r} \cong \mathcal{R}_r(m \times (d-d_s))$. The dimension $r(d - d_s + m - r)$ follows directly from Theorem~\ref{thm:rank_manifold}.
\end{proof}

\begin{remark}
Theorem~\ref{thm:constrained_weight} is not used in the proof of Theorem~\ref{lem:symmetry_llc} or its corollaries. It is a standalone geometric result characterizing the dimension of the rank-$r$ constrained weight manifold when data lies on a lower-dimensional subspace. Its relevance to LoRA~\cite{Hu2021LoRA} settings --- where $W$ is explicitly constrained to rank $r$ --- is a structural observation; we do not experimentally validate this connection here.
\end{remark}

\subsection{Generalization to Affine Transformations}
The results established above apply to strictly linear weight matrices. However, as demonstrated in Section 2.3, neural network layers often include explicit bias terms or operate on affine manifolds (e.g., $f(x) = Wx + b$). The extension to the affine case is immediate via homogeneous coordinates:

\begin{remark}[Affine symmetry space]
Let $\mathcal{M} \subset \mathbb{R}^d$ have affine span $d_a$. Augmenting $x \mapsto \tilde{x} = [x^\top, 1]^\top \in \mathbb{R}^{d+1}$ converts the affine condition $Ux + c = 0$ into the linear condition $\tilde{U}\tilde{x} = 0$ for $\tilde{U} = [U, c]$. The augmented manifold $\tilde{\mathcal{M}}$ has linear span $d_a + 1$, so the rank-nullity argument of the Notation gives
\[ \dim(S_{\mathcal{M}}^{\mathrm{affine}}) = m\bigl((d+1) - (d_a+1)\bigr) = m(d - d_a). \]
\end{remark}

\begin{corollary}[Affine LLC Drop]
\label{cor:affine_llc}
Under the same conditions as Theorem~\ref{lem:symmetry_llc} (single linear layer with bias, squared loss, realizable teacher, non-degenerate input distribution), applied in homogeneous coordinates: let $f(x) = Wx + b$ with $W \in \mathbb{R}^{m \times d}$, $b \in \mathbb{R}^m$, and let the input data lie on a manifold $\mathcal{M} \subset \mathbb{R}^d$ with affine span $d_a$. Then the parameter pair $(W, b)$ admits a continuous symmetry space of dimension $m(d - d_a)$, and
\[ \Delta\lambda = \frac{m(d - d_a)}{2}. \]
\end{corollary}

In particular, \texttt{Softmax} maps its input to the standard simplex, which has affine span $d_a = d-1$, yielding $\Delta\lambda = m/2$ whenever the layer includes an explicit bias. A strictly linear layer ($b = 0$) is unaffected, as the simplex's linear span remains $d$.

\begin{remark}[Symmetry-type matching principle]
Corollaries~\ref{cor:llc_drop} and~\ref{cor:affine_llc} together express a single principle: for data confined to a hyperplane $\{n^\top x = c\} \subset \mathbb{R}^d$, an LLC drop of $m/2$ occurs if and only if the transformation type has at least as much \emph{affinity} as the constraint. A strictly linear layer $Wx$ is symmetry-coupled only when $c = 0$ (the hyperplane passes through the origin); an affine layer $Wx+b$ is symmetry-coupled for any $c$, including $c = 0$. When $c \neq 0$, a strictly linear transformation admits no symmetry because the offset cannot be cancelled without a free bias term. Data not confined to any hyperplane (e.g.\ a sphere) has full affine span and yields $\Delta\lambda = 0$ under either transformation type.
\end{remark}

\section{The Geometric Threshold: Curvature vs.\ Flatness}\label{sec:curvature}

The main theorem proves that affine flatness is sufficient for an LLC drop: a holonomic constraint that eliminates a full dimension of linear span costs exactly $m/2$ from the LLC. This suggests a \emph{binary} threshold --- any non-zero curvature, regardless of sign or magnitude, should be sufficient to preserve the LLC entirely. We test this converse direction by training on manifolds with systematically varied geometry.

\subsection{Setup}
All experiments use a single linear layer $W \in \mathbb{R}^{m \times d}$, MSE loss with a teacher-student setup, Gaussian inputs mapped onto the target manifold, $N=2000$ samples, 5 seeds (median reported). Full experimental details are given in Section~\ref{sec:normalization}. We use $d=5$, $m=4$ throughout, giving a theoretical $\Delta\lambda = m/2 = 2.0$ for flat manifolds. LLC is estimated via SGLD with the devinterp library~\cite{devinterpcode} (3 chains, 4000 burn-in + 4000 recorded draws, localization $\gamma = 2/p$, lr $= 5 \times 10^{-4}$). We compare the LLC against an unconstrained Gaussian baseline ($\lambda_\text{Gaussian} = 12.045 \pm 0.283$).

\subsection{Block A: The Binary Curvature Threshold}
We test three qualitatively different curvature regimes, each a $(d-1)$-dimensional manifold embedded in $\mathbb{R}^d$:
\begin{itemize}
    \item \textbf{Positive curvature:} Paraboloid ($x_d = \|x_{<d}\|^2$) and Hyperboloid ($x_d = \sqrt{1 + \|x_{<d}\|^2}$).
    \item \textbf{Negative curvature:} Saddle ($x_d = x_1^2 + x_2^2 - x_3^2 - x_4^2$).
    \item \textbf{Zero curvature:} Flat hyperplane ($x_d = 0$).
\end{itemize}

Results are shown in Table~\ref{tab:curvature_results} and Figure~\ref{fig:appendix_curvature} (Left). All three curved manifolds yield LLC within $0.15$ of the Gaussian baseline, confirming $\Delta\lambda \approx 0$. The flat hyperplane drops to $9.939 \pm 0.146$, giving $\Delta\lambda = 2.106 \approx m/2 = 2.0$. Curvature sign is irrelevant: what matters is only whether the manifold is flat.

\begin{table}[ht]
    \centering
    \caption{\textbf{LLC by manifold geometry.} $d=5$, $m=4$, $N=2000$, 5 seeds, median $\pm$ std. Theory: curved $\to \Delta\lambda = 0$; flat $\to \Delta\lambda = m/2 = 2.0$.}
    \vspace{2mm}
    \begin{tabular}{lccc}
        \toprule
        \textbf{Manifold} & \textbf{LLC (median $\pm$ std)} & \textbf{$\Delta\lambda$} & \textbf{Predicted} \\
        \midrule
        Gaussian baseline (unconstrained) & $12.045 \pm 0.283$ & --- & --- \\
        Paraboloid (positive curvature)   & $11.950 \pm 0.222$ & $+0.10$ & $0$ \\
        Hyperboloid (positive curvature)  & $11.904 \pm 0.362$ & $+0.14$ & $0$ \\
        Saddle (negative curvature)       & $12.114 \pm 0.162$ & $-0.07$ & $0$ \\
        \textbf{Flat hyperplane}          & $\mathbf{9.939 \pm 0.146}$  & $\mathbf{+2.11}$ & $\mathbf{2.0}$ \\
        \bottomrule
    \end{tabular}
    \label{tab:curvature_results}
\end{table}

\subsection{Block B: Finite-Sample Phase Transition}
Theory predicts the threshold is sharp (binary). In practice, with finite data and SGLD noise, curvature must exceed a critical scale to register as non-flat. We quantify this by perturbing a flat hyperplane with a Gaussian bump: $x_d = A\exp(-\alpha\|x_{<d}\|^2)$, varying amplitude $A$ and width $\alpha$.

Two regimes (Table~\ref{tab:bump_results}, Figure~\ref{fig:appendix_curvature} Right):
\begin{itemize}
    \item \textbf{Wide bump ($\alpha = 0.1$):} curvature is spread across the manifold. LLC rises smoothly with $A$, crossing above the flat bound near $A^* \approx 0.1$.
    \item \textbf{Narrow bump ($\alpha = 10.0$):} curvature is concentrated at the origin; most data points land on the flat region. The LLC remains near the flat bound until $A \geq 1.0$, despite the manifold being theoretically non-flat for any $A > 0$.
\end{itemize}

This characterizes a limitation of SGLD-based LLC estimation: the sampler is blind to geometric structure below its thermal noise floor. The measured LLC is not just a function of whether curvature exists, but of \emph{how much of the data distribution experiences it}. Below a critical amplitude that depends on the data--curvature overlap, the sampler cannot distinguish the manifold from flat. The true RLCT remains binary (as the algebraic theory predicts); the smooth crossover is a property of the estimator, not the geometry.

\begin{table}[ht]
    \centering
    \caption{\textbf{LLC vs.\ Gaussian bump amplitude.} Flat baseline $= 9.939$. $\delta = \text{LLC} - \text{LLC}_\text{flat}$.}
    \vspace{2mm}
    \begin{tabular}{ccccc}
        \toprule
        & \multicolumn{2}{c}{\textbf{Wide ($\alpha=0.1$)}} & \multicolumn{2}{c}{\textbf{Narrow ($\alpha=10.0$)}} \\
        \cmidrule(lr){2-3} \cmidrule(lr){4-5}
        $A$ & LLC & $\delta$ & LLC & $\delta$ \\
        \midrule
        0.01 & 9.941 & $+0.002$ & 9.939 & $+0.000$ \\
        0.10 & 9.982 & $+0.043$ & 9.940 & $+0.001$ \\
        0.20 & 10.073 & $+0.133$ & 9.944 & $+0.004$ \\
        0.50 & 10.540 & $+0.601$ & 9.965 & $+0.026$ \\
        1.00 & 11.261 & $+1.322$ & 10.041 & $+0.102$ \\
        2.00 & 11.710 & $+1.771$ & 10.267 & $+0.328$ \\
        \bottomrule
    \end{tabular}
    \label{tab:bump_results}
\end{table}

\begin{figure}[ht]
    \centering
    \includegraphics[width=\textwidth]{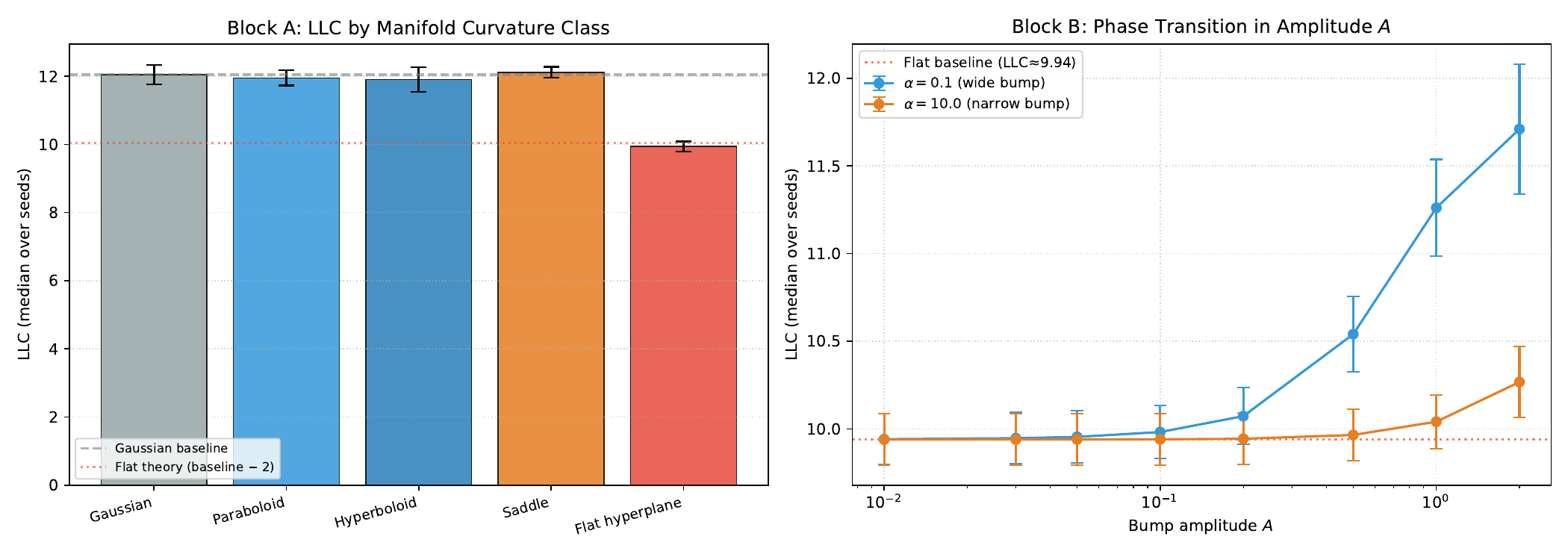}
    \caption{\textbf{Curvature geometry and effective LLC.} \textbf{Left:} Median LLC per manifold class ($d=5$, $m=4$, 5 seeds). All curved surfaces match the Gaussian baseline ($\approx 12.0$); only the flat hyperplane drops to $\approx 9.9 \approx 12.0 - m/2$. Curvature sign is irrelevant: the threshold is between flat and non-flat. \textbf{Right:} LLC vs.\ bump amplitude $A$ for wide ($\alpha=0.1$) and narrow ($\alpha=10.0$) bumps. The wide bump recovers near $A^* \approx 0.1$; the narrow bump stays near the flat bound until $A \geq 1.0$, showing that the effective RLCT depends on data--curvature overlap, not just the existence of curvature.}
    \label{fig:appendix_curvature}
\end{figure}

\section{Application to Normalization Layers}\label{sec:normalization}

To empirically verify the bounds established in Theorem~\ref{lem:symmetry_llc} and Corollary~\ref{cor:llc_drop}, we must isolate the geometric cost of normalization from the inherent degeneracies of neural networks. The key challenge is that a multi-layer network exposes unbounded continuous symmetries — such as the $GL(r)$ scaling invariances inherent in adjacent matrix factorizations — which flood the SGLD sampler and make clean LLC estimation intractable. We eliminate this problem by restricting to a minimal single-layer architecture where the \textit{only} source of degeneracy is the holonomic constraint imposed by the normalization under study.

\subsection{Minimal Single-Layer Architecture}
Our model consists of a single linear layer $W \in \mathbb{R}^{m \times d}$ (no bias) optionally preceded by a normalization operation:
\[ f(x) = W \cdot \phi(x) \]
where $\phi$ is one of three conditions: (1) \textbf{Baseline} ($\phi = \text{identity}$), (2) \textbf{LayerNorm} (zero-mean, unit-variance, no learned affine parameters), or (3) \textbf{RMSNorm} (unit-RMS, no learned affine parameters). Because there is only a single weight matrix and no downstream frozen components, $W$ is the sole source of any structural degeneracy. Any LLC drop measured relative to the Baseline is therefore uniquely attributable to the data manifold's geometric constraint.

\subsection{Teacher-Student Setup}
We construct a Teacher-Student setup to guarantee a perfectly realizable target, ensuring the student can reach a true global minimum ($L \approx 0$). For each condition, a Teacher model with the same architecture is randomly initialized (weights drawn from $\mathcal{N}(0,1)$) and frozen. Inputs $X \in \mathbb{R}^{N \times d}$ ($N=2000$) are drawn from $\mathcal{N}(0, I_d)$ and targets are set to $Y = \text{teacher}(X)$. The Student is then trained on $(X, Y)$ via MSE loss until convergence, ensuring it finds the same zero-loss manifold as the Teacher.

\subsection{Training and Convergence}
The Student model is optimized using Adam~\cite{kingma2014adam} with cosine annealing~\cite{loshchilov2016sgdr} (initial lr $= 10^{-2}$, minimum lr $= 10^{-5}$, 1000 epochs, batch size 256). After training, we verify convergence by checking that the final MSE loss is below $10^{-4}$. Seeds that fail this criterion are flagged as poorly converged; no such failures were observed in the experiments reported here.

\subsection{SGLD \texorpdfstring{$\lambda$}{lambda} Estimation}
After training, we estimate the LLC using SGLD~\cite{10.5555/3104482.3104568} via the \texttt{devinterp} library~\cite{devinterpcode}. To keep the per-parameter spring energy constant regardless of model size, we use an adaptive localization $\gamma = 2.0 / (m \cdot d)$, calibrated to the reference value of $\gamma = 0.1$ that works well for $p = 20$ parameters (since $2/20 = 0.1$). We run 3 chains of 4000 burn-in steps followed by 4000 recorded draws, with SGLD learning rate $5 \times 10^{-4}$.

To suppress the occasional divergent chain, we repeat each configuration over 5 independent seeds and report the \textbf{median} $\Delta\lambda = \lambda_\text{Baseline} - \lambda_\text{norm}$ across seeds, together with the standard deviation as a measure of per-seed variance.

\subsection{Scaling Experiments}
We conduct two systematic sweeps to verify the scaling predictions of Corollary 1:
\begin{itemize}
    \item \textbf{Experiment 1 (vary $m$):} Output dimension $m \in \{2, 4, 6, 8, 10\}$ with input dimension fixed at $d = 12$. Theory predicts $\Delta\lambda = m/2$.
    \item \textbf{Experiment 2 (vary $d$):} Input dimension $d \in \{6, 10, 14, 18\}$ with output dimension fixed at $m = 4$. Theory predicts $\Delta\lambda = 2.0$ regardless of $d$.
\end{itemize}

\subsection{Softmax Degeneracy Experiment}
To empirically validate Corollary~\ref{cor:affine_llc}, we evaluate a separate configuration where the input data $X$ is drawn from the affine hyperplane $\{x \in \mathbb{R}^d \mid \mathbf{1}^\top x = 1\}$ rather than a Gaussian. Concretely, each sample is generated as $x = z - \bar{z}\,\mathbf{1} + \frac{1}{d}\mathbf{1}$ where $z \sim \mathcal{N}(0, I)$, ensuring $\mathbf{1}^\top x = 1$ while preserving the spread of a standard Gaussian. This captures the essential affine geometry of \texttt{Softmax} outputs (which satisfy the same constraint) while keeping the data variance high enough for reliable LLC estimation. We use the same single-layer architecture ($d=5, m=4$) and compare three downstream configurations: (1) strictly linear ($b=0$), (2) explicit bias ($y = Wx + b$), and (3) \texttt{LayerNorm} post-projection. This isolates whether the simplex geometry triggers the ``smuggled bias'' degeneracy predicted by the affine symmetry analysis.

\section{Empirical Results}

\subsection{LayerNorm Reduces LLC by \texorpdfstring{$m/2$}{m/2}}
The LLC drop induced by \texttt{LayerNorm} is consistent with the $m/2$ prediction across all tested output dimensions (Figure~\ref{fig:scaling}, Left), in agreement with Corollary~\ref{cor:llc_drop}. Because \texttt{LayerNorm} mean-centers the data, it projects $x$ onto the hyperplane $\{v \in \mathbb{R}^d \mid \mathbf{1}^\top v = 0\}$, restricting the linear span to $d_s = d-1$ and introducing an $m$-dimensional symmetry space in $W$. Each lost dimension costs exactly $1/2$ from the LLC.

At $m=10$ the estimator breaks down: seeds cluster near $3$ and $7$--$8$ separately, yielding a median $\Delta\lambda = 3.68 \pm 2.07$ against a theoretical prediction of $5.0$. The bimodal distribution indicates the SGLD sampler is not converging to a single stable estimate but settling into two distinct modes in the loss landscape. The theory predicts $\Delta\lambda = 5.0$; neither observed mode matches this. This is an estimator failure at larger parameter counts, not a confirmation of the theory; controlled single-layer experiments at this scale would require more chains, longer burn-in, or better-calibrated localization.

Figure~\ref{fig:scaling} (Right) confirms that the drop is independent of input dimension $d$: fixing $m=4$ gives $\Delta\lambda \approx 2.0$ for $d \in \{6, 10, 14, 18\}$. The LLC reduction depends only on the number of dimensions lost to the constraint, not on the ambient size of the manifold.

\begin{figure}[ht]
	\centering
	\includegraphics[width=\textwidth]{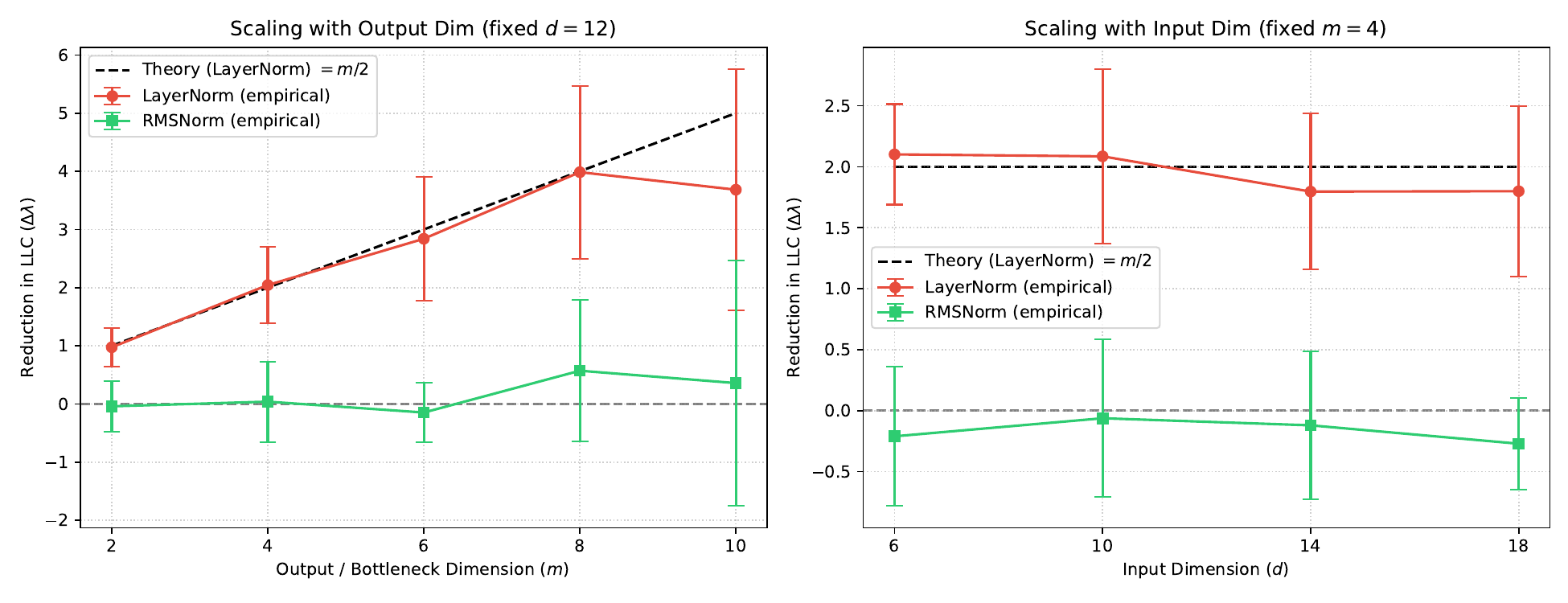}
	\caption{\textbf{LLC reduction under \texttt{LayerNorm} and \texttt{RMSNorm}.} \textbf{Left:} Varying output dimension $m$ with fixed $d=12$. $\Delta\lambda$ is consistent with the $m/2$ prediction for \texttt{LayerNorm} (with increasing variance at large $m$); \texttt{RMSNorm} shows $\Delta\lambda \approx 0$ throughout. \textbf{Right:} Varying input dimension $d$ with fixed $m=4$. $\Delta\lambda \approx 2.0$ regardless of $d$, confirming that the drop depends only on the number of dimensions lost to the constraint.}
	\label{fig:scaling}
\end{figure}

\subsection{RMSNorm Preserves LLC}
\texttt{RMSNorm} shows $\Delta\lambda \approx 0$ across all tested values of $m$ and $d$ (Figure~\ref{fig:scaling}). Without mean-centering, its outputs lie on a hypersphere with $d_s = d$: $W$ has no data-induced null space and no continuous symmetry is introduced. This gives a theoretical account of why \texttt{RMSNorm} provides normalization's training-stability benefits without reducing the effective parameter count of the subsequent weight matrix.

\subsection{Softmax and the Smuggled Bias}
Results are summarized in Table~\ref{tab:softmax_results}. With a strictly linear layer ($b=0$), the simplex data yields $\Delta\lambda \approx 0$: because the simplex does not pass through the origin, its linear span is $d$ and no symmetry space exists. The LLC drop is conditional, not automatic.

Adding an explicit bias activates the degeneracy. With $y = Wx + b$, the measured $\Delta\lambda = 2.26 \approx m/2 = 2.0$: the smuggled bias $b_\text{smuggled} = \frac{1}{d}W\mathbf{1}$ and the explicit bias $b$ are indistinguishable to the loss, collapsing $m$ parameters in accordance with Corollary~\ref{cor:affine_llc}. This result is robust across seeds (std $< 0.5$).

The post-projection \texttt{LayerNorm} case (Conjecture~\ref{conj:postln}) is not experimentally resolved: multi-seed estimation gives $\Delta\lambda = -0.92 \pm 0.86$, indicating that the SGLD-based Gaussian baseline is poorly calibrated for the non-affine \texttt{LN} loss landscape. We omit that row from the table and leave validation to future work.

\begin{table}[ht]
    \centering
    \caption{\textbf{LLC under Softmax simplex geometry.} $d=5$, $m=4$. $\Delta\lambda$ measured relative to Gaussian-input baseline (5-seed median). The simplex alone causes no drop; degeneracy requires a downstream operation sensitive to the affine offset. The Post-\texttt{LayerNorm} case is omitted: SGLD estimation is unreliable for the non-affine \texttt{LN} loss landscape (see Conjecture~\ref{conj:postln}).}
    \vspace{2mm}
    \begin{tabular}{lcccc}
        \toprule
        \textbf{Configuration} & \textbf{Baseline ($\lambda$)} & \textbf{Simplex ($\lambda$)} & \textbf{$\Delta \lambda$} & \textbf{Expected} \\
        \midrule
        Strictly Linear ($b=0$) & 10.38 & 10.62 & $-0.24$ & $0.0$ \\
        Explicit Bias ($+b$)    & 12.40 & 10.14 & $\mathbf{2.26}$ & $m/2 = 2.0$ \\
        \bottomrule
    \end{tabular}
    \label{tab:softmax_results}
\end{table}

\section{Conclusion}

The central finding is a geometric threshold: whether a normalization operation reduces the LLC of the subsequent weight matrix depends entirely on whether the data manifold is affinely flat, not on any continuous geometric quantity like curvature. \texttt{LayerNorm} confines data to a linear hyperplane ($\mathbf{1}^\top x = 0$, $d_s = d-1$), producing an LLC drop of $m/2$ (proved for single linear layers with squared loss); \texttt{RMSNorm} projects onto a sphere, which has full linear span, producing no drop. For the codimension-one manifolds we study, any non-zero curvature is sufficient to preserve the LLC --- the distinction appears binary, not graded.

This binary threshold has a smooth finite-sample counterpart. The effective RLCT measured by SGLD depends not just on whether curvature exists, but on how much of the data distribution actually experiences it. Wide, distributed curvature restores the LLC at amplitudes an order of magnitude smaller than narrow, concentrated curvature --- even when both are theoretically non-flat. This is an intrinsic feature of SGLD-based LLC estimation, not a failure of the algebraic theory; the true RLCT remains binary.

The wrLLC methodology makes the geometric cost of normalization precise and measurable in a single-layer experiment. Under the conditions of Theorem~\ref{lem:symmetry_llc} (single linear layer, squared loss), \texttt{LayerNorm} reduces $\Delta\lambda = m/2$ for any non-degenerate input distribution; \texttt{RMSNorm} preserves LLC by projecting onto a full-rank sphere; \texttt{Softmax} induces the same $m/2$ drop when paired with an explicit bias (proved, experimentally confirmed); the Post-\texttt{LayerNorm} case is a theoretical open problem. A natural open question is whether the LLC drop in $W_V$ under Pre-\texttt{LayerNorm} manifests as stronger attention sinks in practice. Our preliminary experiments on a small attention-only transformer found no significant difference between \texttt{LayerNorm} and \texttt{RMSNorm} conditions, suggesting this connection, if it exists, may require larger models or tasks that more directly expose the low-rank structure of $W_V$.

A further direction concerns \emph{grokking dynamics}: by removing degrees of freedom associated with memorization, \texttt{LayerNorm}'s LLC reduction may make the low-LLC generalizing circuit more accessible during training, potentially accelerating the grokking transition relative to \texttt{RMSNorm} or unnormalized baselines.

\bibliographystyle{unsrtnat}
\bibliography{references}

\appendix
\section{Homology vs.\ Curvature: Guaranteeing Affine Span}\label{app:homology}

In Section~\ref{sec:curvature} we observed empirically that LLC drops for affinely flat manifolds and is preserved for the curved manifolds we tested. Here we give the algebraic-topological explanation for why global topology (non-trivial homology) provides a particularly \emph{robust} guarantee of full linear span, and why this robustness is qualitatively stronger than the finite-sample guarantee provided by local curvature alone.

\texttt{RMSNorm} projects data onto a hypersphere $S^{d-1}$, which possesses a non-trivial top-homology group ($H_{d-1}(S^{d-1}) \neq 0$; see \cite{Hatcher2002}). The linear span of $S^{d-1}$ is full directly: for each standard basis vector $e_i$, both $e_i$ and $-e_i$ lie on $S^{d-1}$, so $\mathrm{span}(S^{d-1}) = \mathbb{R}^d$. Consequently $d_s = d$, and by Corollary~\ref{cor:llc_drop} no LLC is lost.

Conversely, the standard simplex generated by \texttt{Softmax} is contractible and strictly flat. Its affine span is $d-1$ because it lies entirely in the hyperplane $\{\sum_i x_i = 1\}$; unlike the sphere, no point and its antipode both belong to it, so there is no \emph{topological} mechanism forcing full linear span. (The simplex does have full linear span $\mathbb{R}^d$ through its vertices $e_1, \ldots, e_d$; the point is that this span is not guaranteed by any robust global property of the manifold, as it is for the sphere.) What the simplex lacks is not full linear span per se, but an affine structure that avoids the origin: its constraint $\mathbf{1}^\top x = 1$ ($c \neq 0$) is what creates the ``smuggled bias'' phenomenon discussed in Section~\ref{sec:theory}.

\subsection*{Local Curvature vs.\ Global Homology}

Section~\ref{sec:curvature} shows empirically that local curvature is \emph{sufficient} to preserve the LLC: all curved manifolds tested (paraboloid, hyperboloid, saddle) matched the Gaussian baseline, while the flat hyperplane dropped by $\Delta\lambda \approx m/2 = 2.0$ (Table~\ref{tab:curvature_results}). Among the manifolds tested, curvature sign is irrelevant; only flatness matters.

However, the Block B experiment (Table~\ref{tab:bump_results}) shows that local curvature and global homology differ sharply in their \emph{robustness at finite sample sizes}. When curvature is small or concentrated, the SGLD sampler cannot distinguish a weakly curved manifold from a flat one: the effective RLCT collapses to the flat bound even when the manifold is theoretically non-flat. This is a finite-sample phase transition in the estimated RLCT, not a failure of the algebraic theory.

Global non-trivial homology (as in $S^{d-1}$) is immune to this effect. A closed orientable hypersurface with non-trivial top-homology separates $\mathbb{R}^d$ into two components and cannot be embedded in any proper hyperplane --- unlike a curved but non-closed patch, which could in principle be continuously flattened. Furthermore, because normalization operations fix the macroscopic scale of the data (standardizing to unit RMS), the topological volume is locked well above the thermal noise floor of the SGLD sampler. We conjecture that non-trivial top-homology therefore provides an unconditional, scale-independent guarantee of LLC preservation, while local curvature provides only a conditional guarantee that breaks down at finite amplitude. We leave a formal proof of this conjecture to future work.

\end{document}